\documentclass{article} % For LaTeX2e
\usepackage{nips11submit_e,times}
\usepackage{graphicx} % more modern
\usepackage{subfigure} 
\usepackage{amsmath}
\usepackage{amssymb}
\usepackage{amsthm}

\usepackage[numbers]{natbib}

\usepackage{algorithm}
\usepackage{algorithmic}

\newtheorem{theorem}{Theorem}
\newtheorem{lemma}{Lemma}

\usepackage{graphicx}
\usepackage{subfigure}

\def\figstochasticbudget{
\begin{figure*}[t]
\center
\includegraphics[width=\linewidth,natwidth=8.64in,natheight=1.92in]{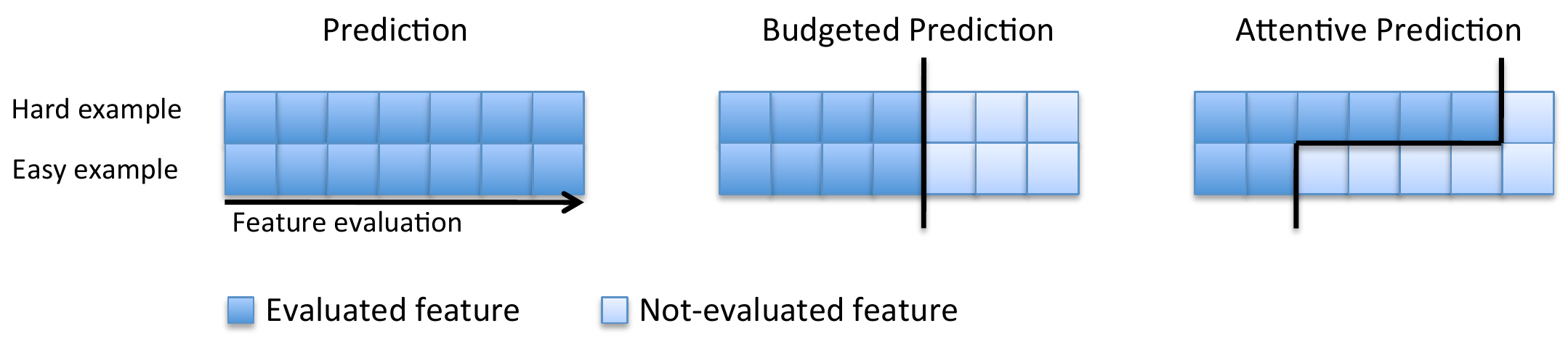}
\caption{Two examples are classified. The first is hard to classify, the second easy. The budgeted learning approach would evaluate the same number of features for both examples, whereas the stochastic would evaluate features according to how hard is the example to classify, while maintaining an average budget of $O(\sqrt{n})$ features.}
\label{fig:stochastic-budget}
\end{figure*}
}

\def\figbbdecisionerror{
\begin{figure*}[t]
\centering
\mbox{
\subfigure[A simulation of the Constant-STST boundary with $X_i \sim N(0.05,1)$. The required decision error rate is on the X-axis, and the actual is on Y-axis. Since applying the Constant-STST results in lower error rates than required, we observe that the boundary is conservative.]{\includegraphics[width=2.7in,height=2.7in]{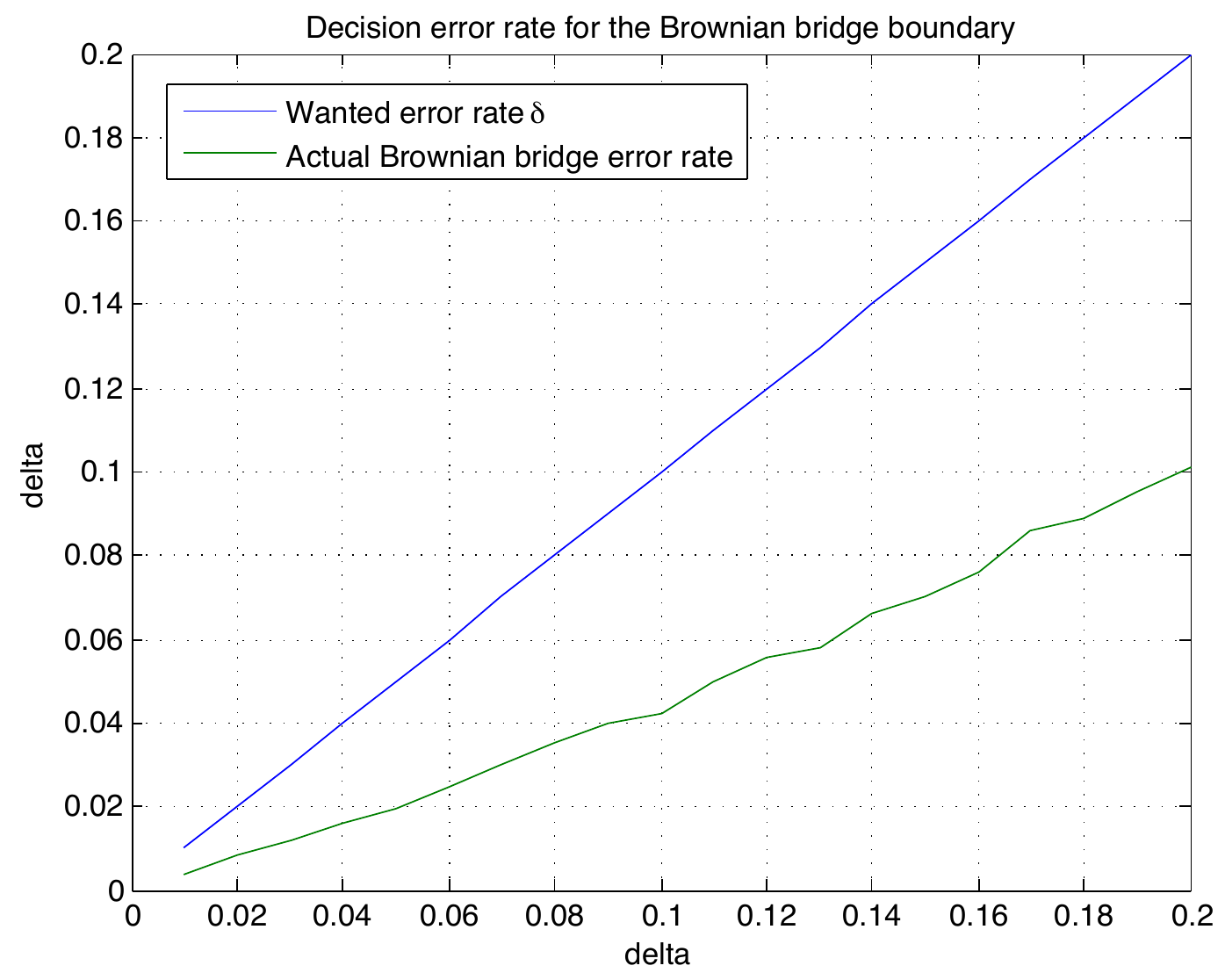} \label{fig:bb_error_rate}
} 
\quad
\subfigure[The boundary behaves similarly to what is expected from theory. It computes in the order of $O(\sqrt  n)$ features.]{\includegraphics[width=2.7in,height=2.7in]{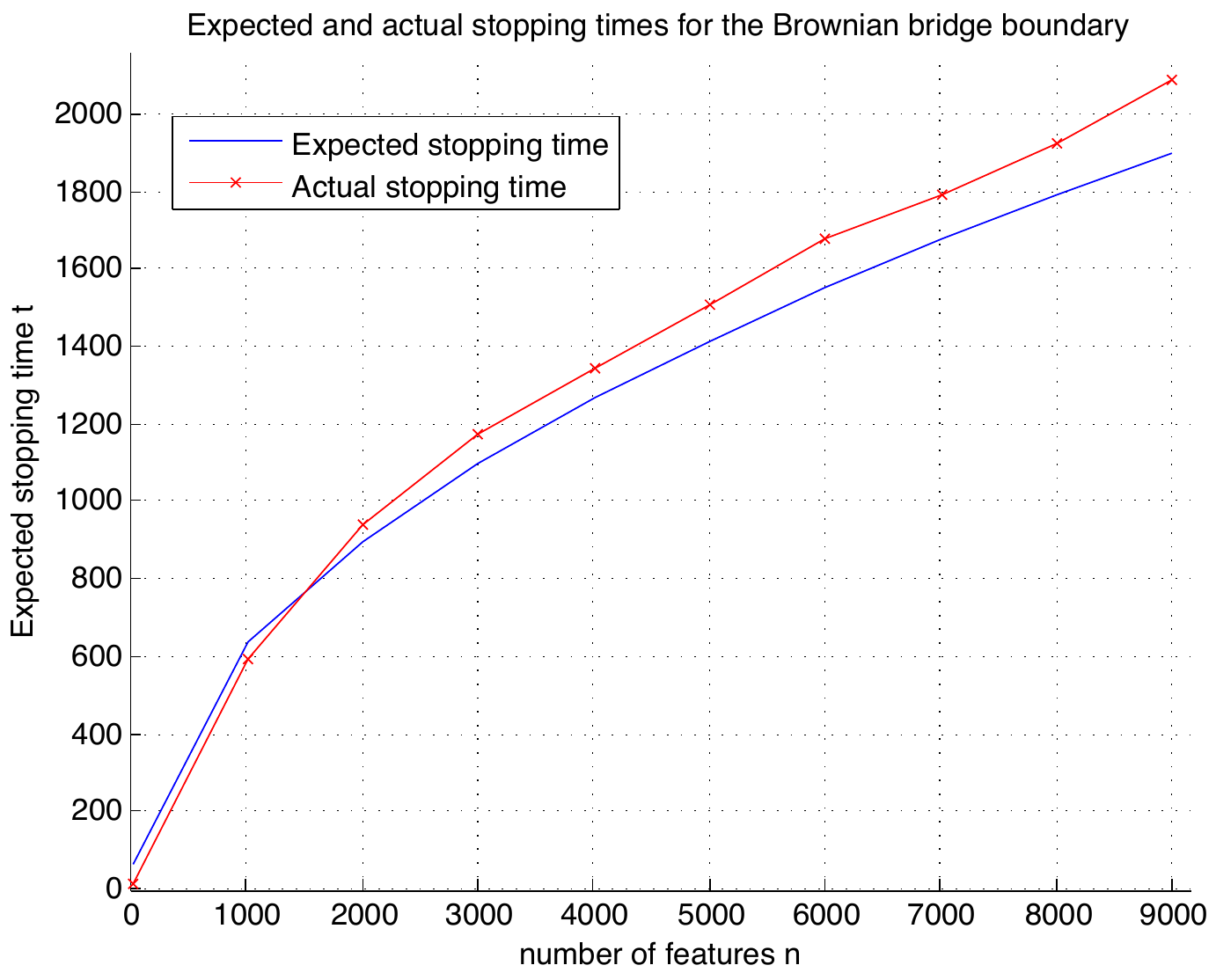} \label{fig:bb_stopping_time}
}
}
\caption{Performance of the Brownian bridge boundary.}
\end{figure*}
}

\title{Focus of Attention for Linear Predictors}

\author{
Raphael Pelossof \\
Computational Biology\\
Memorial Sloan Kettering Cancer Center\\
New York, NY 10065 \\
\texttt{pelossof@cbio.mskcc.org} \\
\And
Zhiliang Ying \\
Statistics Department \\
Columbia University \\
New York, NY 10027 \\
\texttt{zying@stat.columbia.edu}
}

\nipsfinalcopy % Uncomment for camera-ready version

\begin{document}

\maketitle

\begin{abstract} 
We present a method to stop the evaluation of a prediction process when the result of the full evaluation is obvious. This trait is highly desirable in prediction tasks where a predictor evaluates all its features for every example in large datasets. We observe that some examples are easier to classify than others, a phenomenon which is characterized by the event when most of the features agree on the class of an example.
By stopping the feature evaluation when encountering an easy-to-classify example, the predictor can achieve substantial gains in computation. 
Our method provides a natural attention mechanism for linear predictors where the predictor concentrates most of its computation on hard-to-classify examples and quickly discards easy-to-classify ones. By modifying a linear prediction algorithm such as an SVM or AdaBoost to include our attentive method we prove that the average number of features computed is $O(\sqrt{n \log \delta^{-0.5}})$ where $n$ is the original number of features, and $\delta$ is the error rate incurred due to early stopping. We demonstrate the effectiveness of Attentive Prediction on MNIST, Real-sim, Gisette, and synthetic datasets.
\end{abstract} 

\section{Introduction}
We wish to avoid evaluating all the weak hypotheses for each example, such that we evaluate less features for easy-to-classify examples. However, to filter an example we need to know whether it is informative or not. The majority vote is used to measure how important an example is for learning, this is proportional to the magnitude of the majority vote. When there is a strong agreement between the features then the majority vote will have larger magnitude than when the features disagree. Our goal is to compute the least number of weak hypotheses possible before we decide whether the majority vote will end below or above an importance threshold. 
Filtering out un-informative examples, and trying to compute as few hypotheses as possible are closely related problems \cite{blum1997selection}. 

The intuition behind this work is prevalent in many natural decision making domains. For example, in finance, suppose we are interested in buying a certain stock, and we have 10 financial advisors at our disposal. If we sequentially ask them whether to buy the stock, and the first four consecutive advisors agree we should buy, we may stop, and decide to buy. With low probability we might incur an error since the remaining six advisors might all vote in the opposite direction. If on the other hand the advisors vote opposing to each other  "yes, no, yes, no...", we will not gain enough confidence in either direction, and will end up asking all of them. Such is the case also in medicine, when the doctor tries to diagnose whether we have a certain condition, he will put us through a sequence of tests. If they all come out negative, then he will stop, and diagnose that we are healthy (with a certain error rate), if the are all positive, he will stop and diagnose that we have the particular condition. However, if the test oppose each other, he will keep on testing with more refined (and probably more expensive) tests. Finally, in face detection, the same type of attention mechanism also holds. In their seminal work \citep{viola01rapid} proposed an attentional cascade, where the face detector could stop the evaluation of the classifier if it did not find any of the important features. The thinking was, if there are no eyes, no nose, no mouth and no ears, don't look for a chin. Attention is achieved in these cases by limiting the amount of computation/time to examples that are easy to classify, and increasing computation time for harder-to-classify examples.

All these decision making problems fall under the same roof of an underlying attention mechanism that stops computation when a the result of the full feature evaluation will produce with high likelihood the result of the partial evaluation, ie. an obvious positive or negative example. We argue that by thresholding the partial evaluation of all votes at each point with a constant, predicting the label we obtain once the partial evaluation first hit the constant stopping threshold, we obtain an attention mechanism that is similar in nature to the reasoning behind these three examples.

%The decision making algorithms that we deal with in this work make a decision by comparing a sum of weighted features to a threshold. If the sum is smaller than a pre-defined prediction threshold then the prediction is negative, otherwise it is positive. This type of a classification rule is prevalent in the margin-based Machine Learning community, where typically an additive model is compared to a threshold, and a subsequent prediction is made depending on the result of the thresholding. Margin-based learning algorithms average multiple weak hypotheses to form one strong combined hypothesis - the majority vote. When testing, the combined hypothesis is compared to a threshold to make a predictive decision about the class of the evaluated example.

%Majority vote based decision making can be generalized to comparing a weighted sum of random variables to a given threshold. Since, the majority vote is a summation of weighted random variables it can be computed sequentially. Sequential Analysis allows us to develop the statistical test needed to stop this evaluation process when its result is evident. 
%The running time of linear predictors depends on the number of features evaluated and the size of the dataset. Since models today may have thousands of features or support vectors, running time may be significant, even when the examples are sparse and prediction may be parallelized. Therefore, one may wish to speed-up prediction by filtering easy-to-classify examples. 
We propose to early stop the computation of feature evaluations for these examples by connecting Sequential Analysis \citep{wald45tests,lan82sequential} and Brownian motion analysis to margin based learning algorithms.

We use the terms margin and full margin to describe the summation of all the feature evaluations, and partial margin as the summation of a part of the feature evaluations.
The calculation of the margin is broken up for each example in the dataset. This break-up allows the algorithm to make a decision after the evaluation of each feature whether the next feature should also be evaluated or the feature evaluation should be stopped, and the label predicted.
By making a decision after each evaluation we are able to early stop the evaluation of features on examples with a large partial margin after having evaluated only a few features. Examples with a large partial margin are unlikely to have a full margin below the required threshold. Therefore, by rejecting these examples early, large savings in computation are achieved. This is quite different from the budgeted approach where a constant smaller number of features is evaluated for all examples, in which case, all examples are treated equally.

Instead of looking at the classification error we look at stop-error. Stop-errors occur when the algorithm stops the partial feature evaluation and predicts a label that is opposite to the label if would have predicted if it had evaluated all the features. 
We demonstrate that a simple rule, comparing each partial sum to a constant stopping-threshold, can speed-up a linear predictor while maintaining generalization accuracy. 

This paper proposes a simple novel test based on Sequential Analysis and stopping methods for Brownian motion to drastically improve the computational efficiency of margin based learning algorithms. Our method accurately stops the evaluation of the margin when the result is the entire summation is evident. Furthermore, our novel algorithm can be easily parallelized. 
\figstochasticbudget
\begin{figure}[t]
\begin{center}
\includegraphics[width=0.6\textwidth]{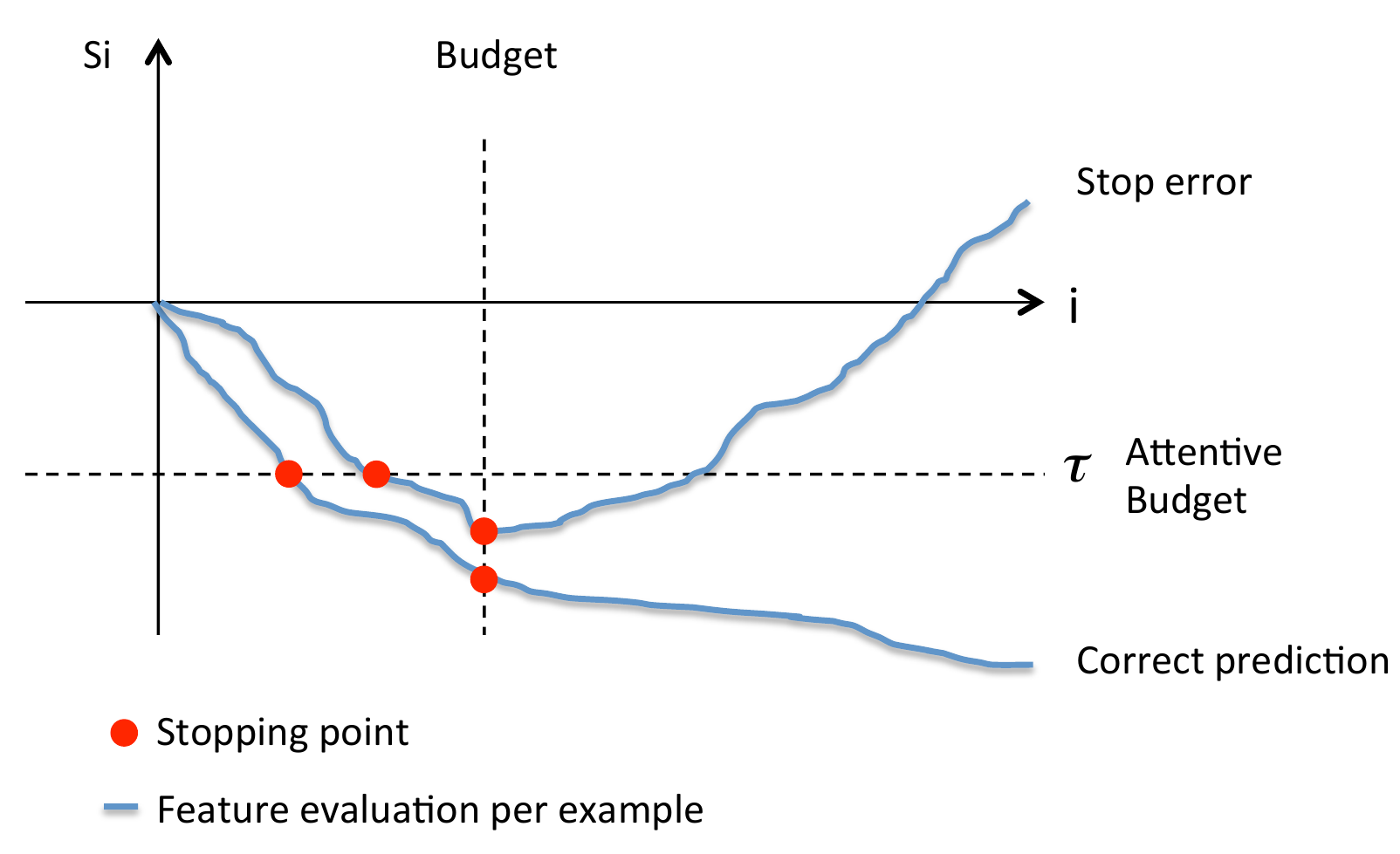}
\caption{Example of the differences between Budgeted prediction and Attentive prediction. Attentive prediction thresholds partial scores, whereas budgeted prediction thresholds the number of features evaluated.}
\label{fig:budget}
\end{center}
\end{figure}
\section{Related Work}
Margin based learning has spurred countless algorithms in many different disciplines and domains. %Typically a margin based learning algorithm evaluates the sign of the margin of each example and performs a decision. Our work provides early stopping rules for the margin evaluation when the result of the full evaluation is obvious. This approach lowers the average number of features evaluated for each example according to its importance.
%Our stopping thresholds apply to the majority of margin based learning algorithms. 
The most directly applicable machine learning algorithms are margin based online learning algorithms. 
Many margin based Online Algorithms base their model update on the margin of each example in the stream. Online algorithms such as Exponentiated Gradient \cite{kivinen97exponentiated} and Online Boosting \cite{oza01online} update their respective models by using a margin based potential function. Passive online algorithms, such as the Perceptron \cite{rosenblatt58perceptron} and online passive-aggressive algorithms \cite{crammer06online}, define a margin based filtering criterion for update, which only updates the algorithm's model if the value of the margin falls below a defined threshold. All these algorithms fully evaluate the margin for each example, which means that they evaluate all their features for every example. Recently Shalev-Shwartz et al. \cite{shalev2008svm,shalev2010pegasos} proposed Pegasos, an online SVM solver. The solver is a stochastic gradient descent solver which produces a maximum margin classifier at the end of the training process.

%The above mentioned algorithms passively evaluate all the features for each given example in the stream. 
However, if there is a cost (such as time) assigned to a feature evaluation we would like to design an efficient learner which actively choose which features it would like to evaluate. Similar work on the idea of learning with a feature budget was first introduced to the machine learning community by Ben-David and Dichterman \cite{bendavid1998learning}. The authors introduced a formal framework for the analysis of learning algorithm with restrictions on the amount of information it can extract. Specifically allowing the learner to access a fixed amount of attributes, which is smaller than the entire set of attributes. Very recently, both Cesa-Bianchi et al. \cite{cesabianchi2010efficient} and Reyzin \cite{reyzin10boosting} studied how to efficiently learn a linear predictor under a feature budget (see figures \ref{fig:stochastic-budget} and \ref{fig:budget}.) Also Clarkson et al.  \cite{clarkson2010sublinear}  extended the Perceptron algorithm to efficiently learn a classifier in sub-linear time. 

Similar active learning algorithms were developed in the context of when to pay for a label (as opposed to an attribute). Such active learning algorithms are presented with a set of unlabeled examples and decide which examples labels to query at a cost. The algorithm's task is to pay for labels as little as possible while achieving specified accuracy and reliability rates \citep{dasgupta05analysis, cesabianchi06sampling}. Typically, for selective sampling active learning algorithms the algorithm would ignore examples that are easy to classify, and pay for labels for harder to classify examples that are close to the decision boundary.

Our work stems from connecting the underlying ideas between these two active learning domains, attribute querying and label querying. The main idea is that typically an algorithm should not query many attributes for examples that are easy to classify. The labels for such examples, in the label query active learning setting, are typically not queried. For such examples most of the attributes would agree to the classification of the example, and therefore the algorithm need not evaluate too many before deciding their importance. 

\section{The Sequential Thresholded Sum Test}
The novel \textit{Constant Sequential Thresholded Sum Test} is a test which is designed to control the rate of stop-errors a margin based learning algorithm makes. This section describes its adaptation.
\figbbdecisionerror
\subsection{Mathematical Roadmap}
Our task is to find a filtering framework that would speed-up margin-based prediction algorithms by quickly classifying obvious examples. Quick classification is done by creating a test that stops the score evaluation process given the partial computation of the score. We measure the difficulty in classifying an example by the magnitude of its score. We define $\theta$ as the prediction threshold, where examples that are predicted as negative have a score smaller than $\theta$ and the rest are predicted as positive. Statistically, this problem can be generalized to finding a test for early stopping the computation of a partial sum of weighted independent random variables when the result of the full summation is guaranteed with a high probability. Given a required decision error rate we will derive the \textit{Constant Sequential Thresholded Sum Test} (Constant-STST) that will provide a constant early stopping threshold that maintains the required confidence.

Let the sum of weighted independent random variables $(X_i,
i=1,...,n)$ be defined by $S_n = \sum_{i=1}^{n} w_i X_i$, where $w_i$
is the weight assigned to the random variable $X_i$. We require that
$w_i\in R, X_i\in[-1,1]$.
We define $S_n$ as the full sum, $S_i$ as the partial sum.
%, and
%$S_{in} = S_n-S_i = \sum_{j=i+1}^n w_i X_i$ as the remaining sum.
Once we computed the partial sum up to the $i$th random variable we
know its value $S_i$. Let the stopping threshold at coordinate $i$
be defined by $\tau_i$.  
%We use the notation $ES_{in}$ to denote the expected value the remaining sum.

Pelossof et al. \cite{pelossof2010speeding} previously proposed a Curved-STST by looking at the following conditional probability
\begin{equation}
P(S_n > \theta | \textit{stop before n}) = \frac{P(S_n > \theta,\textit{stop before n}) }{P(\textit{stop before n})},
\end{equation}
where the event ``stop before n'' is the event which occurs when the partial sum crosses a stopping boundary $\textit{stop}\doteq \{S_i < \tau_i\}$ at any point $i$ along the stopping curve, yielding the prediction $S_n \le \theta$. 
The simplicity of deriving the curtailed method stems from the fact that the joint probability and the stopping time probability are not needed to be explicitly calculated to upper bound this conditional. The resulting curved stopping boundary, gives a constant conditional error probability throughout the curve, which means that it is a rather conservative boundary. 

However, if we are interested of controlling stop errors for a given set of examples, we are interested in a slightly different conditional probability $P(\text{stop before } n | S_n < \theta)$. Such is the case in many classification tasks where there are significantly more negatives than positives in the dataset.  This formulation results in a more aggressive boundary which allows higher stop error rates at the beginning of the evaluation and lower stop error rates at the end. Such a boundary stops more evaluations early on, and less later later on. This approach has the natural interpretation that we want to shorten the feature evaluation for obvious negative samples, but we want to prolong the evaluations for positive samples. A Constant boundary achieves this exact ``error spending'' characteristic. Note that this equation can be flipped in order to stop the evaluation of positive examples as well.

\subsection{The Constant Sequential Thresholded Sum Test}
We condition the probability of making a decision error in the following way
\begin{eqnarray}
P(\text{stop before } n | S_n > \theta) = \frac{P(\text{stop before } n , S_n > \theta)}{P(S_n > \theta)}.
\label{eqn:rare-decision-probability}
\end{eqnarray}
We stated in equation \ref{eqn:rare-decision-probability} a
conditional probability function which is conditioned on the
examples of interest. Therefore in this case we are interested in
limiting the stop error rate for examples that are important.
To upper bound this conditional we will make an approximation that
will allow us to apply boundary-crossing probability estimation
for a Brownian bridge. To apply the Brownian bridge to our
conditional probability we note that
\begin{eqnarray}
P(\text{stopped before } n  | S_n > \theta) 
= P(\max_i S_i < \tau | S_n > \theta) 
\approx  P(\min_i S_i < \tau | S_n = \theta).
\end{eqnarray}
where $\tau<0$ is a constant stop threshold. The last approximation holds when the event $\{S_n > \theta\}$ is
rare, i.e.  $EX_i<0$, and $n$ is large, so that the event is
concentrated on $S_n$ being close to $\theta$. 
Now we can
approximate the stop error rate by calculating the corresponding
boundary crossing probability.

\begin{lemma}
\label{thm:bb-upper} Let $T_\tau = \inf\{i:S_i \ge \tau\}$ be the
first crossing time of the random walk over constant $\tau$. Then
the probability of the following decision error,
$P(T_\tau<n|S_n=\theta)$ is approximately equal to $
e^{-\frac{2\tau(\tau-\theta)}{var(S_n)}}$ when $n$ is large.
\end{lemma}
\begin{proof}
By the Functional Central Limit Theorem, we know that
$S_{[tn]}/\sqrt {S_n}$, $t\in [0, 1]$, converges to the Brownian
Motion process. The conclusion of the lemma then follows from
Lemma 2 in the Appendix.
\end{proof}

\begin{theorem}
(Suppose that $\theta = 0$). Choose $\tau =
\sqrt{var(S_n)}\sqrt{log\frac{1}{\sqrt{\delta}}}$ for the constant
boundary. Then the rate for decision error $\{T_\tau < n | S_n <
0\}$ is approximately $\delta$.
\end{theorem}
\begin{proof}
The Theorem follows from Lemma \ref{thm:bb-upper} directly by plugging in
$\theta=0$ and $\tau=\sqrt{var(S_n)}\sqrt{log\frac{1}{\sqrt{\delta}}}$.
\end{proof}

When using this boundary for prediction, we can directly see the implication on the error rate to the classifier, since the error rate of an attentive predictor is equal to at most the error rate of the full predictor plus the stop error rate.

\subsection{Average Stopping Time for the Curved and Constant-STST}
We now show that the expected number of features evaluated for the
Curved and the Constant STST boundaries is in the order of
$O(\sqrt{n \log \delta^{-0.5}})$. This is obtained by limiting the range of values
$X_i$ can take. Without loss of generality, the theorem is proved for the case where we early stop the computation of positive predictions, and can be applied also for the mirrored case of early stopping negative predictions.
\begin{theorem}
\label{thm:stopping-time}
Suppose that random variables $X_i$ are
bounded, i.e. $|X_i| \le k$ for a constant $k$, and that $EX_i
> 0$. Let the stopping time be defined by
$T = \inf \{i: S_i \ge \sqrt{var(S_n)\log \delta^{-0.5}} \}$. Then
the expected stopping time is of order $O(\sqrt n \log\delta^{-0.5})$.
%Let $|X_i| \le k$, and let $EX_i > 0$. Let the stopping time of the Brownian bridge be defined by $t = \inf \{i: S_i \ge \sqrt{var(S_n)\log \delta^{-0.5}} \}$. Then the expected stopping time is in the order of $O(\sqrt n)$.
\end{theorem}
\begin{proof}
\begin{equation}
ES_T = ES_{T-1} + EX_T
\le ES_{T-1} + k
\le \sqrt{var(S_n)\log \delta^{-0.5}} + k
\end{equation}

The second inequality holds since the random walk only crossed the boundary for the first time at time $T$ and therefore was under the boundary at time $T-1$. By applying Wald's Equation $ES_T = ET EX$  \cite{wald44cumulative}, we get
\begin{equation}
ET = \frac{\sqrt{var(S_n)\log \delta^{-0.5}} + k}{EX} 
%\le \frac{c \sqrt{n} \sqrt{\log \frac{1}{\sqrt\delta}} + k}{EX} 
= O(\sqrt n),
\end{equation}
where $c,k$, and $EX$ are constants.
\end{proof}
See figure \ref{fig:bb_error_rate} for simulation results.

\section{Experiments with Attentive Prediction}
We apply the above stopping rule to linear prediction. For the different classification tasks 'Gisette', 'MNIST 2 vs. 5', and 'Real-sim' we treat the data in the same way. We split the data to training and test sets, train an SVM classifier on a training set, and predict on the held out test set. We calculate the mean kernel value for each support vector on the positive test set $\mu_i = E_lK_{il}(X_i,X_l)$ (l is an index over positive examples, and i is an index over support vectors), and store it as constants. When predicting we evaluate the corrected score as $S_t = \sum_{i=1}^t \alpha_i (K(X_i,X) - \mu_i)$, which removes the trend from the positive set. In our experiments to induce independence between the support vectors we randomly permute the order of the support vectors and then calculate the partial scores. In experiments where we sorted the support vectors by $|\alpha_i|$ we observed unstable results. 

We conducted three experiments to test the loss of the Attentive approach versus a Budgeted approach and the Full classifier. We tested the algorithm on the Gisette dataset which includes 6,000 training examples, 1,000 test examples, and 5,000 features. An SVM was trained on this dataset with a linear kernel, and $C=1$, resulting in a model with 1084 support vectors. Our second dataset was MNIST digits 2 vs. 5, which comprised of $28x28$ images of the number 2 and the number 5. The training set included 11,379 images, and the test set 1,924. We trained an SVM with an RBF kernel with $\sigma=0.1, C=1$ and obtained a model with 781 support vectors. Finally we trained an SVM on the Real-sim dataset which is comprised of 72,309 examples and 20,958 dimensions. The data was split to 2,000 randomly chosen test examples, and the rest were used for training. For training we used a linear kernel with C=1. The results of the full training can be seen in figure \ref{fig:results} top row in red. We obtained highly accurate classifiers. The to compare Attentive prediction with Budgeted prediction we thresholded the partial scores at a certain threshold, obtained a confusion matrix for that threshold, as well as the average number of support vectors evaluated. We then ran Budgeted prediction on the exact same set with a budget set to the average number of support vectors evaluated by the Attentive predictor and obtained a comparable confusion matrix. Since the bounds as seen by the simulation are not tight we tested all possible thresholds, setting $\tau \in \{\min S_{il},..,\theta\}$ for the entire dataset. We also set the prediction threshold $\theta=0$. We tested for each example in the test set $AttentivePredict(K(.,X)-\mu, \alpha, \tau)$, where $K(.,X)-\mu$ is the centered Kernel evaluation, and $\alpha$ is a vector of weights obtained by the SVM. Similarly we tested Budgeted prediction, with a budget equal to the average number of features calculated by AttentivePredict over the entire test set $BudgetedPredict(K(.,X)-\mu,\alpha, b_{\tau}$).

\begin{algorithm}[h]
\begin{algorithmic}
\caption{Attentive Prediction}
\label{alg:att_pred}
\STATE AttentivePredict(X,w, $\tau$)
\IF {$\exists i: \sum_{i=1}^n w_i X_i < \tau$} 
	\STATE return $\tau$
\ELSE
	\STATE return $\sum_{i=1}^n w_i X_i$
\ENDIF
\end{algorithmic}
\end{algorithm}
\begin{algorithm}[h]
\caption{Budgeted Prediction}
\label{alg:bug_pred}
\begin{algorithmic}
\STATE BudgetPredict(X,w, b)
\STATE return $\sum_{i=1}^b w_i X_i$ 
\end{algorithmic}
\end{algorithm}

The results of the early stopping algorithms can be seen in figure \ref{fig:results}. The top row shows the precision recall for the different algorithms over the three sets. AttentivePredict outperforms BudgetedPredict on all tasks. The reason for this is probably since the attentive threshold is only applied as a one sided test, therefore it is likely to reject mostly negative examples and classify them as negatives, thereby lowering the false positive rate of this classifier. Negative examples that may have been predicted by full feature evaluation as positive, now have a chance of being rejected at an interim point. Furthermore, since the distribution of partial scores for the positives is typically above that of the negatives, one can set up an attentive threshold that positive examples would never reach, and negative examples might hit, thereby improving the FP rate at no expense. From the bottom part of the figure we can see that indeed the Attentive predictor produces lower stop error rates than the budgeted predictor at the low range of under $30\%$ stop error rate. Beyond that, the Attentive threshold gets very close to the positive distribution at the beginning of the feature evaluation, and is too aggressive. A more conservative attentive predictor can be used at this range. We observe that the attentive predictor can significantly lower the number of SVs evaluated (up to 50\% less) without much loss in predictive power.

\subsection{Conclusions}
We sped up prediction algorithms up to 50\% without significant loss in predictive accuracy. We proved that the expected speedup under independence assumptions of the weak hypotheses is $O(\sqrt{n\log\delta^{-0.5}})$ where $n$ is the set of all features used by the learner for discrimination.

As a future direction we wish to study the performance of such algorithms when the independence assumptions are broken, by sorting data by feature weight or sampling according to other measures of importance.

The thresholding process creates a natural attention mechanism for linear predictors. Examples that are easy to classify (such as background) are filtered quickly without evaluating many of their features. On the other hand examples that are hard to classify, the majority of their features are evaluated. By spending little computation on easy examples and a lot of computation on hard ``interesting'' examples the Attentive algorithms exhibit a stochastic focus-of-attention mechanism.

\renewcommand{\thesubfigure}{\relax}

\begin{figure}[tbp]
\centering

\subfigure[]{
   \includegraphics[width=0.3\textwidth] {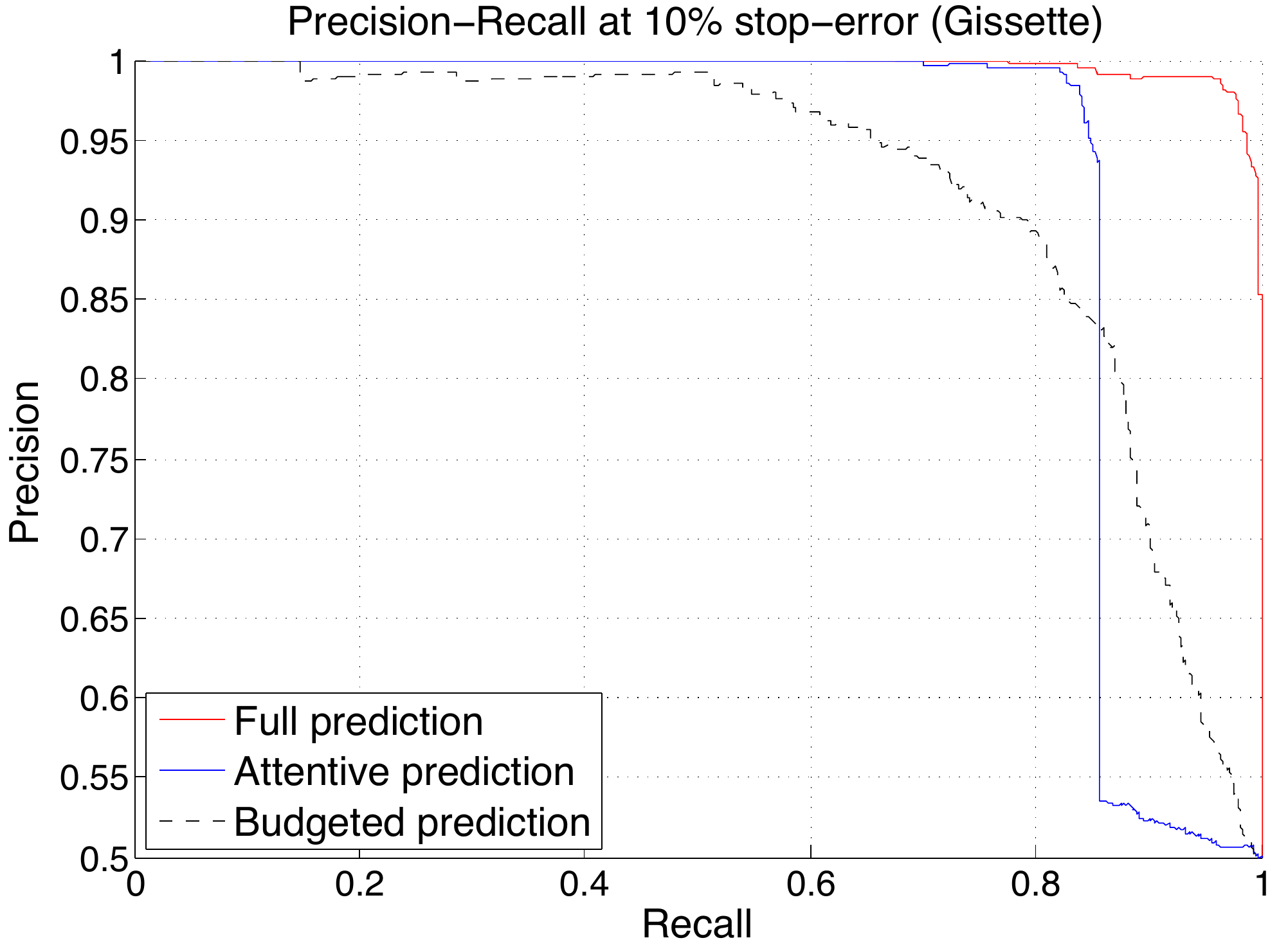}
   \label{fig:subfig1}
 }
 \subfigure[]{
   \includegraphics[width=0.3\textwidth] {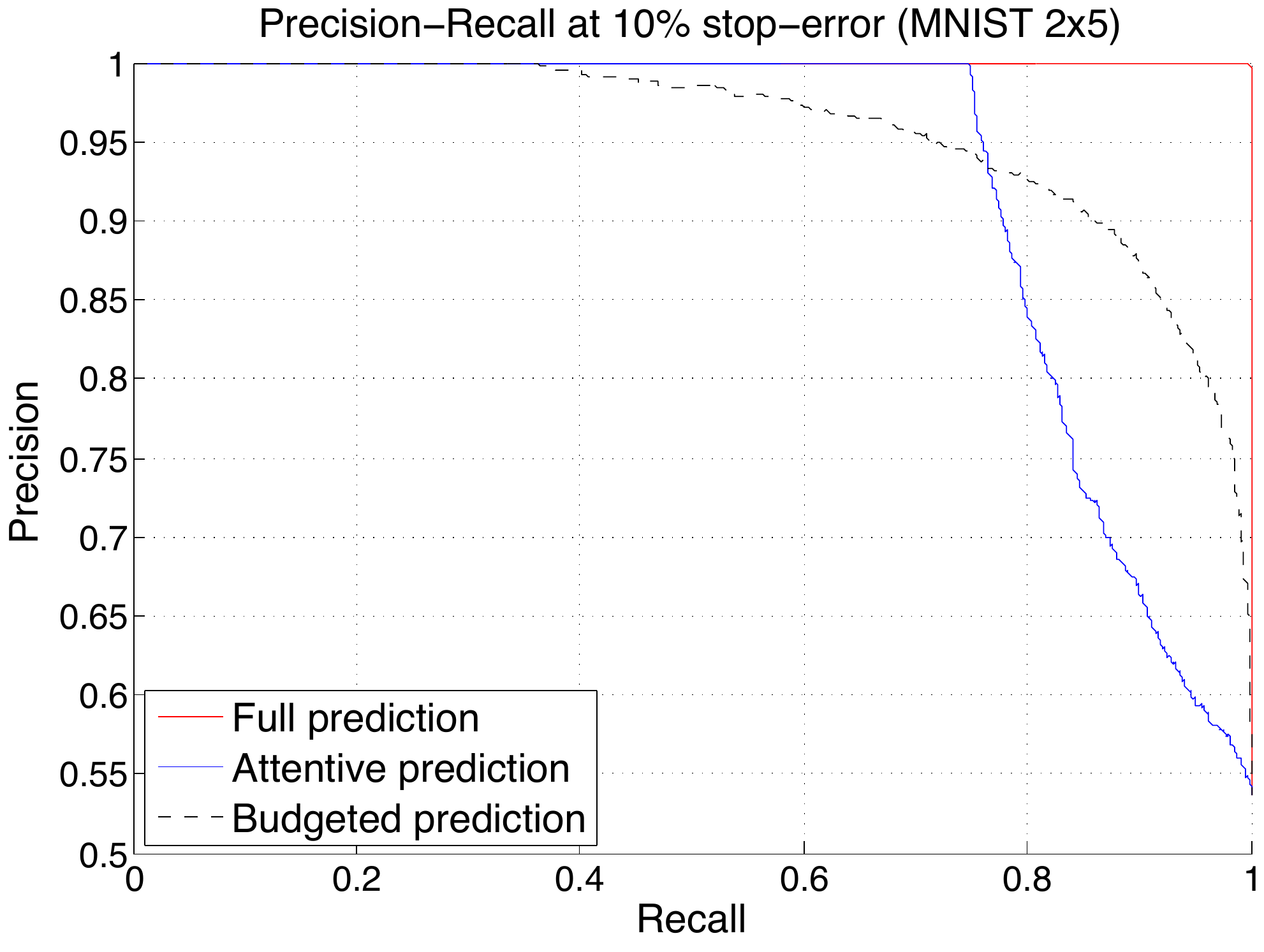}
   \label{fig:subfig2}
 }
 \subfigure[]{
   \includegraphics[width=0.3\textwidth] {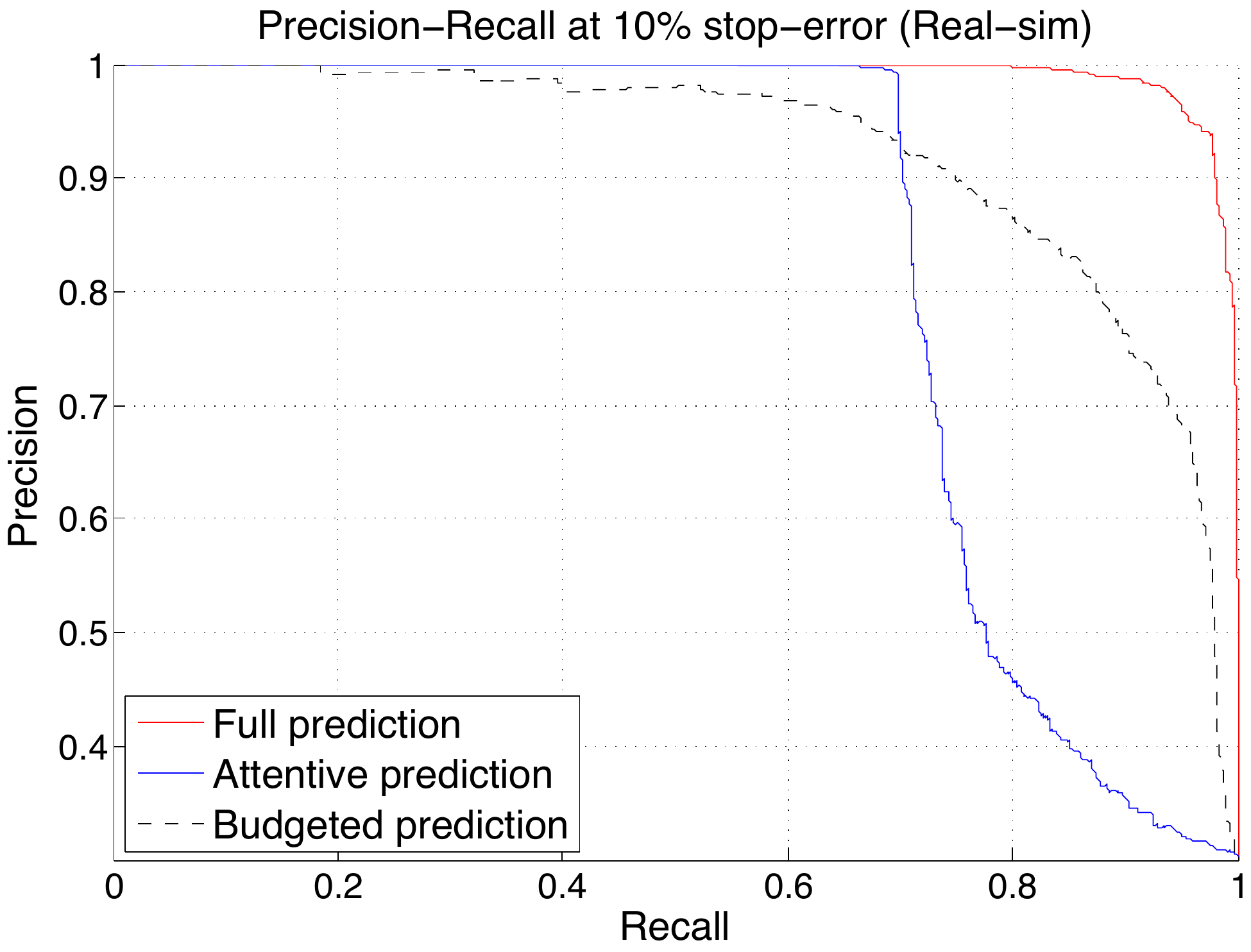}
   \label{fig:subfig3}
 }\\
\subfigure[]{
   \includegraphics[width=0.3\textwidth] {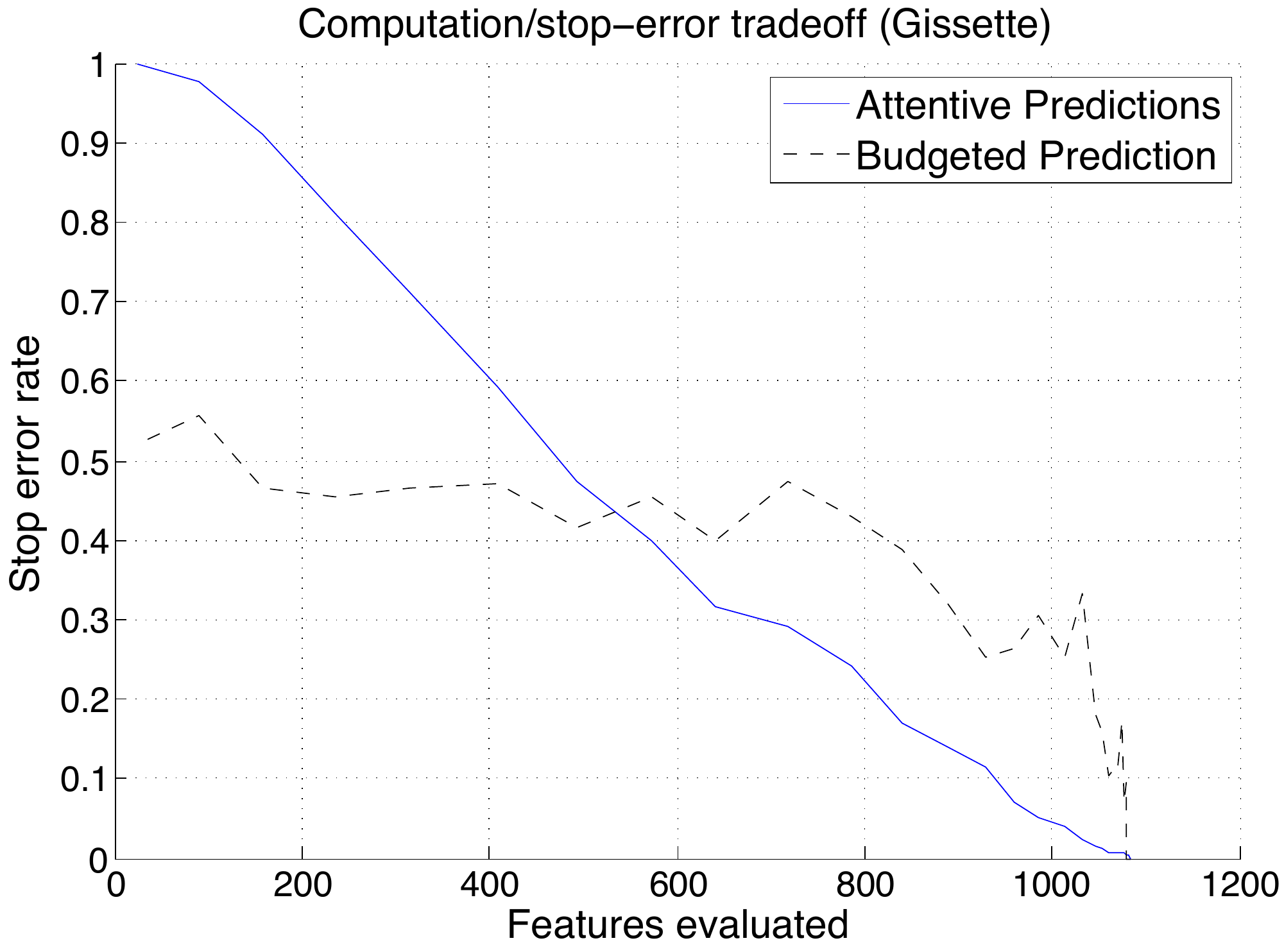}
   \label{fig:subfig1}
 }
 \subfigure[]{
   \includegraphics[width=0.3\textwidth] {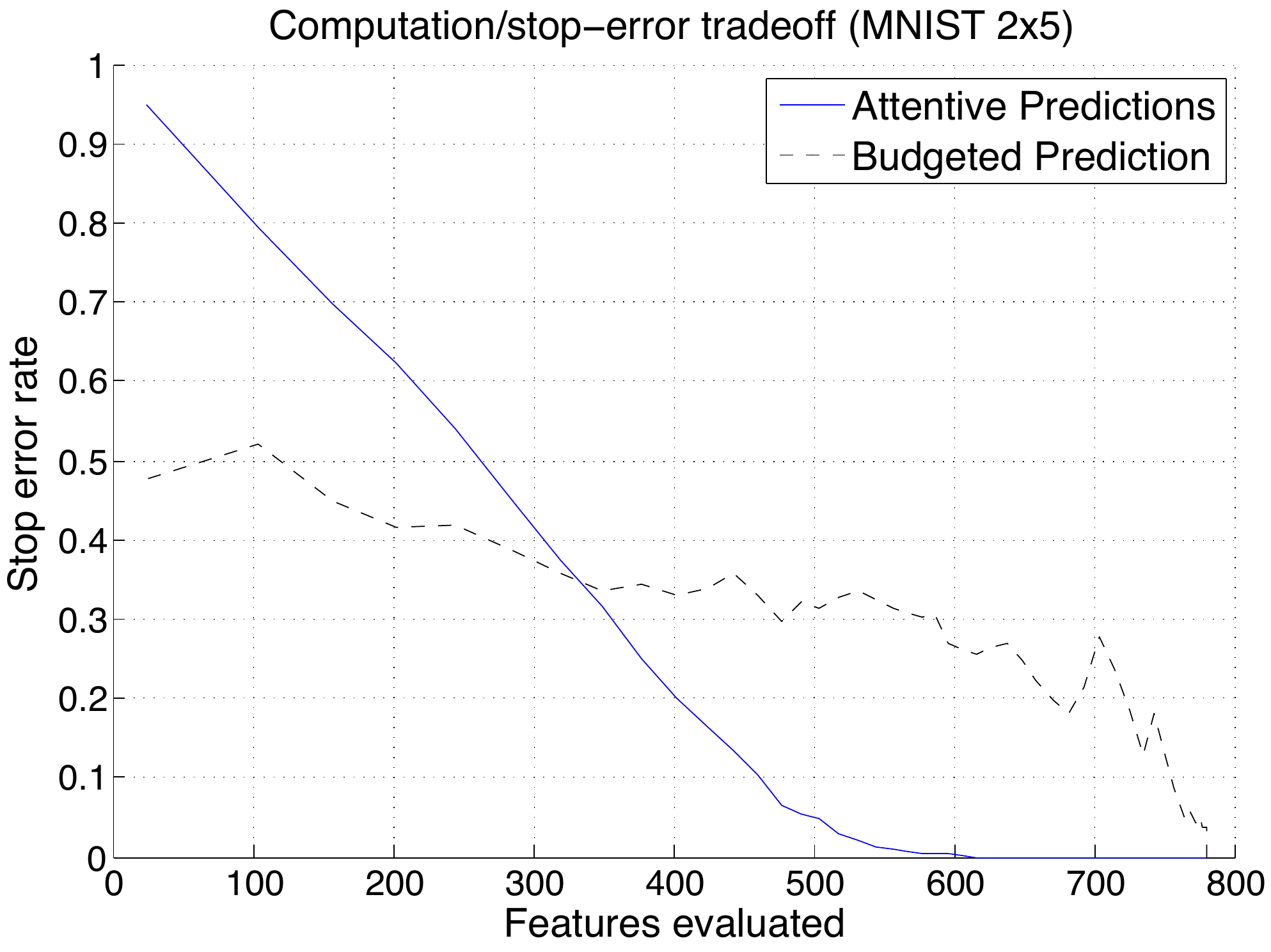}
   \label{fig:subfig2}
 }
 \subfigure[]{
   \includegraphics[width=0.3\textwidth] {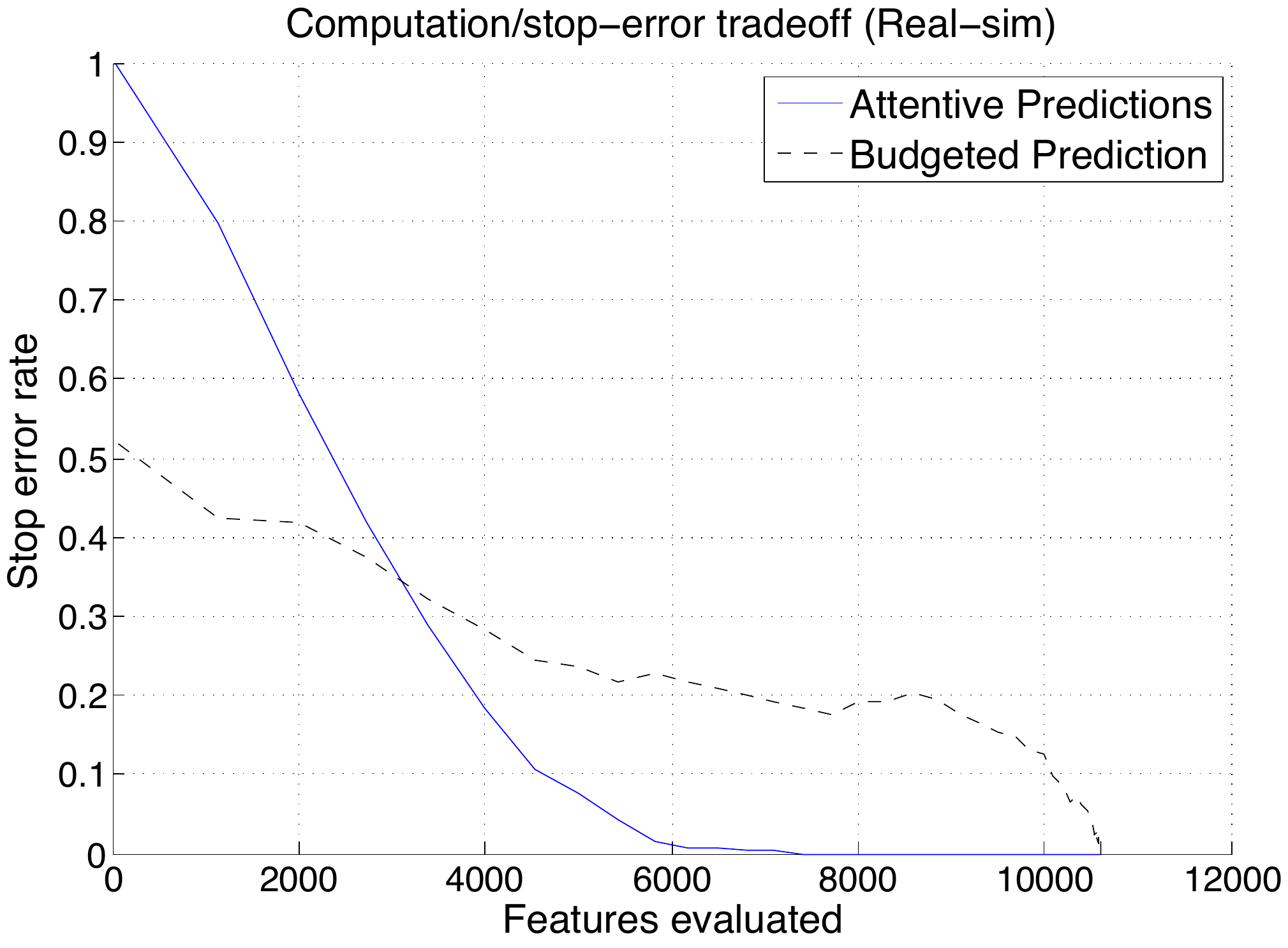}
   \label{fig:subfig3}
 }

\label{fig:results}
\caption{Comparison of Attentive Prediction with Budgeted Prediction. Top, Precision-recall for the actual performance of the end classifier. Bottom, computation/stop-error tradeoff. Attentive prediction outperforms Budgeted Prediction in all classification tasks. Since the Attentive predictor thresholds negatives at a constant, these predictions are void of magnitude and there is a steep fall off in the precision recall plot at that point. In terms of the stop-error/computation tradeoff, Attentive prediction requires less computation than Budgeted prediction if the required error rate is less than 30\%. }
\end{figure}

\section{Appendix }
\label{sec:appendix}

\begin{lemma}
Let $S_u$, $u\ge 0$ be a continuous time Brownian motion process
and $T_\tau=\inf \{u: S_u\ge \tau\}$. Then, for $\theta<\tau$ and
$t>0$,
\begin{equation}
P(T_\tau < t | S_t = \theta)=\exp\left\{
-\frac{2\tau(\tau-\theta)}{var(S_t)} \right\}.
\end{equation}
\end{lemma}
\begin{proof}
We can look at an infinitesimally small area $d\theta$ around
$\theta$. Then, by definition of conditional probability,
\begin{equation}
P(T_\tau \le t | S_t = \theta) = \frac{P(T_\tau < t , S_t\in
d\theta)}{P(S_t\in d\theta)}, \label{eqn:bb-lim}
\end{equation}
where $S_t\in d\theta$ denotes $S_t\in [\theta, \theta+d\theta)$.
For the numerator, we have, by the reflection principle,
\begin{eqnarray*}
P(T_\tau < t , S_t\in d\theta)
%&=& P(T_\tau < n) P(S_n \in d\theta | T_\tau<n) \\
%&=& P(T_\tau < n) P(S_n \in 2\tau-d\theta | T_\tau<n) \\
&=& P( T_\tau < t, S_t \in 2\tau-d\theta)
= P(S_t \in 2\tau-d\theta) \\
&=& \frac{1}{\sqrt{var(S_t)}}\phi\left(
\frac{2\tau-\theta}{\sqrt{var(S_t)}}\right)d\theta,
\end{eqnarray*}
where $\phi$ is the standard normal density function. But we
certainly know that for the denominator
\[P(S_t\in d\theta)= \frac{1}{\sqrt{var(S_n)}}
\phi\left( \frac{\theta}{\sqrt{var(S_n)}}\right)d\theta.\]
Plugging the preceding two equalities back into \ref{eqn:bb-lim},
we get
\begin{eqnarray*}
P(T_\tau < n | S_n = \theta)
&=& \frac{\phi\left( \frac{2\tau-\theta}{\sqrt{var(S_t)}}\right)}{\phi\left( \frac{\theta}{\sqrt{var(S_t)}}\right)} 
= \exp\left\{ -\frac{1}{2} \frac{(2\tau-\theta)^2}{{var(S_t)}} + \frac{1}{2}\frac{\theta^2}{{var(S_t)}} \right\} \\
&=& \exp\left\{ -\frac{2\tau(\tau-\theta)}{{var(S_t)}} \right\}.
\end{eqnarray*}
\end{proof}

\small{

%\bibliography{bibdesk}
\bibliographystyle{plain}
}

\end{document}